\documentclass{article}

\usepackage[nonatbib,final]{neurips_2024}

\usepackage[utf8]{inputenc} 
\usepackage[T1]{fontenc}    
\usepackage{hyperref}       
\usepackage{url}            
\usepackage{booktabs}       
\usepackage{amsfonts}       
\usepackage{nicefrac}       
\usepackage{microtype}      
\usepackage{xcolor}         
\usepackage{graphicx}
\usepackage{amsmath}
\usepackage{subcaption}
\usepackage{multirow}
\usepackage{booktabs}
\usepackage{makecell}
\usepackage{float}
\usepackage{caption} 
\usepackage{algorithm}
\usepackage{algpseudocode}
\usepackage{setspace}
\usepackage{amsthm}
\usepackage{verbatim}
\usepackage{colortbl}

\newtheorem{theorem}{Theorem}[section]
\newtheorem{definition}{Definition}[section]
\newtheorem{lemma}{Lemma}[section]

\title{RCDM: Enabling Robustness for Conditional Diffusion Model}

\author{Weifeng~Xu\\
  College of Meteorology and Oceanography,
  National University of Defense Technology \\
  Changsha, CHN 410000 \\
  \texttt{xuweifeng.23@nudt.edu.cn} \\
   \And
   Xiang~Zhu \\
   College of Meteorology and Oceanography, National University of Defense Technology \\
   Changsha, CHN 410000 \\
   \texttt{zhuxiang@nudt.edu.cn} \\
   \And
   Xiaoyong Li \\
   College of Meteorology and Oceanography, National University of Defense Technology \\
   Changsha, CHN 410000 \\
   \texttt{sayingxmu@nudt.edu.cn} \\
}

\begin{document}

\maketitle

\begin{abstract}
The conditional diffusion model~(CDM) enhances the standard diffusion model by providing more control, improving the quality and relevance of the outputs, and making the model adaptable to a wider range of complex tasks. However, inaccurate conditional inputs in the inverse process of CDM can easily lead to generating fixed errors in the neural network, which diminishes the adaptability of a well-trained model. The existing methods like data augmentation, adversarial training, robust optimization can improve the robustness, while they often face challenges such as high computational complexity, limited applicability to unknown perturbations, and increased training difficulty. In this paper, we propose a lightweight solution, the Robust Conditional Diffusion Model (RCDM), based on control theory to dynamically reduce the impact of noise and significantly enhance the model's robustness. RCDM leverages the collaborative interaction between two neural networks, along with optimal control strategies derived from control theory, to optimize the weights of two networks during the sampling process. Unlike conventional techniques, RCDM establishes a mathematical relationship between fixed errors and the weights of the two neural networks without incurring additional computational overhead.  Extensive experiments were conducted on MNIST and CIFAR-10 datasets, and the results demonstrate the effectiveness and adaptability of our proposed model. Code and project site: \href{https://github.com/*****/RCDM}{https://github.com/*****/RCDM}.
\end{abstract}

\section{Introduction}
Conditional diffusion models (CDMs)~\cite{dhariwal2021diffusion, fu2024unveil, ho2022classifier} extend traditional diffusion models~\cite{ho2020denoising} by integrating additional guidance, such as labels or text, to direct the data generation process. Unlike unguided diffusion generating data through noise addition and denoising, conditional models produce more targeted outputs.
These models excel at generating data that matches given criteria, improving accuracy. They are applicable, in handling tasks like computer vision~\cite{ho2022cascaded, 2024arXiv240219481L, saxena2024surprising}, predicting time series~\cite{2024arXiv240103001C,rasul2021autoregressive}, and tackling complex interdisciplinary challenges~\cite{shi2021learning,wang2024multi,xu2024geometric,xu2022geodiff}. Their flexibility and precision make them a powerful tool in diverse domains.

However, conditional diffusion models face severe challenges in real-world scenarios due to noisy data~\cite{chen2024overview,xie2024robust}. Neural networks often experience performance degradation in the presence of noise and outliers, which limits their reliability. Robust learning techniques aim to address this issue by enhancing adaptability to data disturbances~\cite{blau2022threat,9156938,wang2022guided,wu2022guided}. The denoising property of diffusion models during the reverse process has been leveraged for adversarial purification~\cite{carlini2022certified,lin2024robust,2024arXiv240414309L}, partially replacing traditional adversarial training~\cite{2024arXiv240408980C, 2024arXiv240112461Y, yoon2021adversarial}, and thus enhancing model robustness. Nevertheless, this approach has limitations when applied to unknown perturbations. Researchers are exploring Distributionally Robust Optimization (DRO)    ~\cite{delage2010distributionally,kuhn2019wasserstein,lin2022distributionally} for deeper robustness assurance by considering comprehensive uncertainty in data distributions. However, implementing DRO may require more computational resources. Complementing DRO, enhancing the intrinsic robustness of  conditional diffusion models through more resilient model structures can mitigate computational overhead and reduce adversarial example generation, providing inherent resistance to input errors.

This study utilizes the framework of deep reinforcement learning~\cite{black2023training, van2016deep}, employing the collaborative interaction of two neural networks to enhance model robustness. Concurrently, it incorporates optimal control strategies~\cite{wang2024adaptive}  from control theory, optimizing the weights of the two neural networks through theoretical analysis and validation. Extensive experiments demonstrate the effectiveness and adaptability of our proposals across diverse datasets and network architectures.


\section{Related Work}

\textbf{Conditional diffusion models.} CDMs~\cite{niu2024acdmsr, yang2024lossy, zhang2023adding} are a significant advancement in diffusion models, enabling the generation of data based on specific inputs for greater control over the generative process. CDMs integrates conditioning information into the training process of a neural network, enabling the model to learn the underlying data distribution specific to a given condition (e.g., class label or text description). It allows for controlled data generation during the model's reverse process. The model consists of two processes, the forward process and the reverse process, as shown in Figure~\ref{fig.1.1}. 
Therefore, CDMs have been applied innovatively to incorporate conditional inputs during neural network training~\cite{dhariwal2021diffusion,  ho2022classifier, liu2023more}. They utilize unconditional networks to modulate the diversity and relevance of the outcomes, which has been a key aspect of their success. 
The Denoising Diffusion Probabilistic Models (DDPM)~\cite{ho2020denoising,nichol2021improved,sohl2015deep} is a prominent research framework for CDMs. It works alongside other frameworks such as score-based generative models (SGM)~\cite{song2020improved} and stochastic differential equations (SDE)~\cite{song2020score}, which are interconnected and complementary to each other.Advancements in CDMs have been made through various improvements, including enhancements in neural network architecture~\cite{bao2023all, gao2023masked, peebles2023scalable}, efficient sampling~\cite{dockhorn2021score,salimans2022progressive}, likelihood maximization~\cite{huang2021variational, lu2022maximum}, and integration with other generative models~\cite{bortoli2022riemannian,jo2022score}. For instance, researchers have worked on scalable architectures, masked autoencoders, and other innovations that have pushed the boundaries of what is possible with these models. These improvements have been documented in various studies, contributing to the ongoing development and refinement of CDMs.
The specific mathematical mechanism of CDMs is shown in Appendix~\ref{Diffusion model}.
\begin{figure}[!ht]
	\centering
	\includegraphics[width=0.9\textwidth]{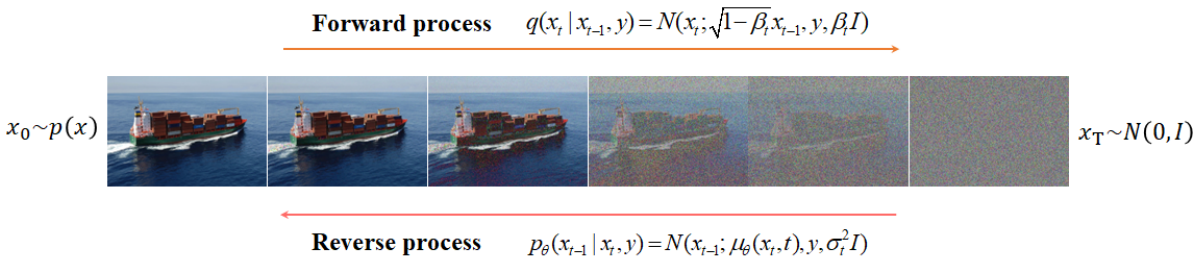}
	\caption{The forward and reverse processes of conditional diffusion models.} 
	\label{fig.1.1}
\end{figure}

\textbf{Robust learning.} In deep learning, model robustness is a critical research goal to enhance their resistance against input disturbances, noise, and adversarial attacks. Researchers used many strategies like data augmentation~\cite{zheng2024improving,zhou2023adversarial} for better adaptability through attention mechanisms~\cite{zhang2024federated}, regularization~\cite{han2024selective, hao2024adversarially}, and network architecture design~\cite{kengne2024robust}.
Among these strategies, diffusion models have emerged as a promising approach due to their ability to denoise data during the adversarial purification process~\cite{carlini2022certified,lin2024robust,2024arXiv240414309L}, thereby minimizing the negative impact of data errors on downstream tasks. However, diffusion models have inherent limitations when it comes to robustness. For instance, they struggle to handle erroneous inputs during the inference phase in image generation tasks. It highlights the need for further research and optimization to enhance their robustness.
Therefore, researchers are actively exploring various techniques and approaches to enhance the robustness of deep learning models, aiming to improve their resistance against different challenges and ensure reliable performance in real-world applications.

\textbf{PID control theory.} The Proportional Integral Derivative (PID) control theory is a cornerstone concept in automatic control, which achieves effective control of a system through the combination of proportional, integral, and derivative parameters. This theory has been extensively utilized across traditional domains, including industrial process control~\cite{liu2023general}, robotics and automation~\cite{chotikunnan2023optimizing}, and aerospace~\cite{gun2023attitude}, while also undergoing continuous evolution and expansion in tandem with technological advancements. For instance, adaptive PID control~\cite{ghamari2023lyapunov} can automatically adjust parameters based on changes in system performance, while fuzzy PID control~\cite{han2023fuzzy} and neural network PID control optimize control~\cite{yakout2024neural} performance by incorporating fuzzy logic and neural network technologies. Furthermore, intelligent PID control enhances the adaptability and robustness of the controller by integrating technologies such as machine learning. Despite facing challenges such as model uncertainty, multivariable control, and networked control, research, and applications of PID control theory continue to develop, demonstrating strong vitality and broad prospects.

\section{Method}

\subsection{Problem Analysis}
The trained neural network in the forward process enters into the reverse process with the following inference process.
\begin{equation}
	x_{t-1} = \frac{1}{\sqrt{\alpha_{t}}} \Bigg[ 
	x_{t} - \frac{1 - \alpha_{t}}{\sqrt{1 - \alpha_{t}}} 
	\Big((1 + w)\varepsilon_{\theta}(x_{t}, y, t) - w\varepsilon_{\theta}(x_{t}, t)\Big) 
	\Bigg] + \sigma_{t} z
	\label{1.4}
\end{equation}

The mathematical mechanism underlying the reverse sampling process in CDMs involves a trained conditional neural network and an unconditional neural network, denoted as $\varepsilon_\theta\left(x_t,y,t\right)$ and $\varepsilon_\theta\left(x_t,t\right)$, respectively. The parameter $w$ is manually set to regulate the weighting between the two networks, thus balancing the diversity and relevance of the generated outcomes.

However, when the two well-trained neural networks encounter inaccurate input data during the sampling phase, it can result in a fixed error in the neural network's output. The concern arises regarding whether this error $\Delta$, accumulating over steps $T$, could potentially lead to uncontrollable effects in the final generation of the diffusion model. 
 \begin{theorem}
	Let $\beta_T=\sqrt{1-\alpha_T}$ and $\overline{\beta}_T=\sqrt{1-\overline{\alpha}_T}$. Due to the inaccuracy of the input data, the error $\Delta$ generated by the neural network during the sampling process can accumulate more than $-\frac{1-a_T}{\sqrt{\overline{\alpha}_T(1-\overline{\alpha}_T)}}\Delta$   after $T$ steps of iterations, which leads to uncontrolled model performance.
	\label{theorem 3.1}
\end{theorem}
\begin{proof}
	We utilize mathematical induction to examine the accumulation of the final error $\Delta$. The detail proof  is provided in Appendix~\ref{Proof for Theorem 3.1}.
\end{proof}
Based on the aforementioned theorem, we have determined that noise from inaccurate inputs is gradually amplified through reverse sampling, adversely affecting the results of the conditional diffusion model. Our objective is to mitigate the amplification and propagation of fixed errors during the sampling process, thereby enhancing the robustness of the conditional diffusion model. 

\subsection{Robust Conditional Diffusion Model Framework}
Our proposed model, RCDM, enhances the performance through a continuous two-stage strategy optimization, which integrates the sequential execution of forward and reverse processes. Initially, the model undergoes iterative training through the forward process, accumulating progressively deepened learning experiences. Subsequently, upon reaching a preset threshold of iterations, it transitions seamlessly to the reverse process for sampling generation, as detailed in Figure~\ref{fig.1}.

\begin{figure}
	\centering
	\hspace{-0.0cm}\includegraphics[width=1.0\textwidth]{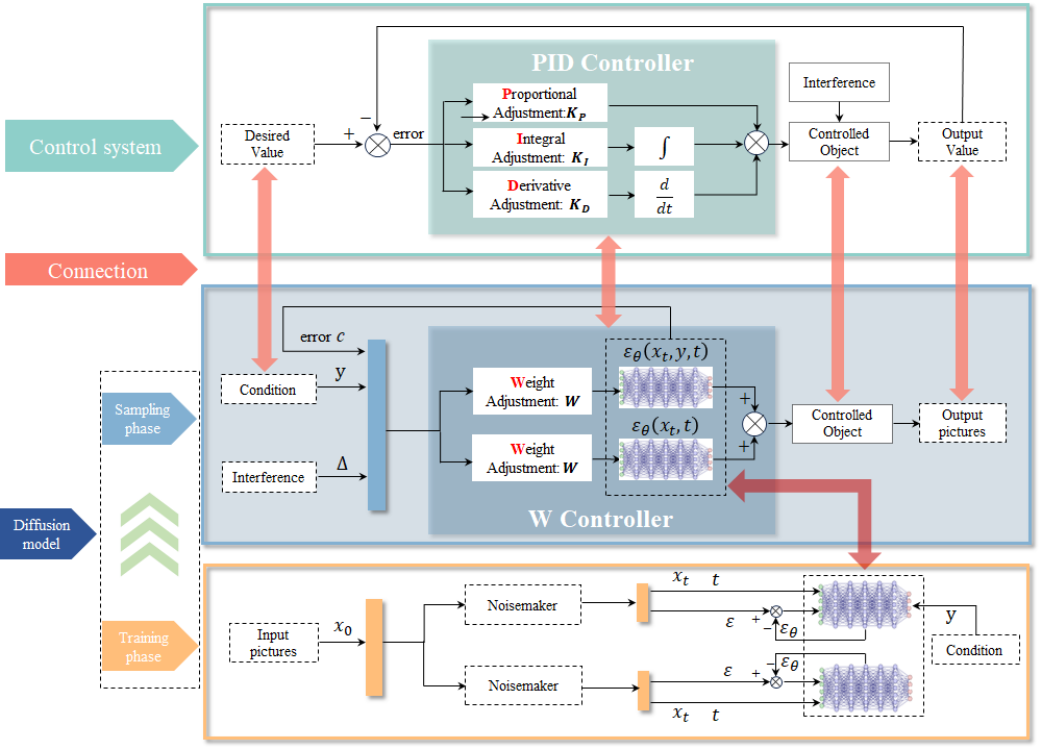}
	\caption{Robust conditional diffusion modeling framework. }\medskip
	\label{fig.1}
			
	\end{figure}
	
In the forward phase of a diffusion model, the input image $x_{0}$ undergoes progressive noise addition to generate a collection $\{x_{t},t,\varepsilon\}$, $\{x_{t},t,\varepsilon,y\}$ that trains both unconditional $\varepsilon_\theta(x_t,t)$ and conditional $\varepsilon_\theta(x_t,y,t)$ neural networks. After sufficient training, the process shifts from the forward phase to the reverse phase, which is also known as the sampling phase. However, inaccuracies in the input data can lead to fixed errors $\Delta $ in the neural network, which can affect the final generation of the image. Therefore, RCDM enhances the ability to resist disturbances by adjusting the parameters $ w $ of the two network weights, drawing on the PID control method from control theory. Ultimately, the image is output through the joint iterative action of the Gaussian distribution $z$ and the control signal $\left[(1+w)\varepsilon_{\theta}(x_{t},y,t)-w\varepsilon_{\theta}(x_{t},t)\right]$.

In the architecture of a robust conditional diffusion model, the design of the filter is one of the key innovations for enhancing the model's robustness. The role of the filter is to optimize the data reconstruction quality in the reverse process by mitigating the impact of errors. In this paper, we concentrate on the derivation of the optimal filter parameters $w$ through a rigorous theoretical framework and associated proof. The objective is to mitigate the effects of input errors inherent in the reverse process, thereby improving the robustness of the data reconstruction.

\subsection{Mathematical Modeling}
In this part, we will utilize mathematics to characterize the distributional properties of neural networks fitted during the forward process of diffusion models. This study seeks to clarify the discrepancies arising in the outputs of the trained network when erroneous input data are introduced during the generative phase. Furthermore, it aims to investigate the impact of these discrepancies on the model's overall performance and ability to generalize. Then, we will mainly study the extent to which specific parameters decline the perturbations caused by input errors during the generation, to gain a profound understanding of the behavioral mechanisms of neural networks in the presence of input data discrepancies and their subsequent impact on model outputs.
\begin{definition}
Assume that $\varepsilon_1$ and $\varepsilon_2$  represent the truths, that is, the Gaussian noise added at a step during the forward process, such that $\varepsilon_{1}\sim(0,I)$ and $\varepsilon_{2}\sim(0,I)$. Here, $\Delta_1$  and $\Delta_2$ denote the discrepancies, which are the errors generated by the neural network architecture during the learning process. And the inaccuracy of the input data in the sampling phase causes the neural network to produce a fixed error $\Delta$.
\end{definition}
 \begin{theorem}
 	The control signal  $\left[\left((1+w)\varepsilon_{\theta}(x_{t},y,t)-w\varepsilon_{\theta}(x_{t},t)\right)\right]$ conforms to distribution $O\sim\left[(1+w)\Delta_1-w\Delta_2,((1+w)^2+w^2)I\right]$,  $O_L\sim\left[0,((1+w)^2+w^2)I\right]$ in the ideal case, and $O_G\sim\left[(1+w)\Delta_1-w\Delta_2+\Delta,((1+w)^2+w^2)I\right]$ in the presence of interference.
 	\label{theorem 4.1}
 \end{theorem}
\begin{proof}
	The training functions of the two networks ($\varepsilon_\theta(x_t,t)$,$\varepsilon_\theta(x_t,y,t)$) obtained from the forward process training are:
	\begin{equation}
		\mathbb{E}_{{x}_0,{y}\sim\tilde{p}({x}_0,{y}),{\varepsilon}\sim\mathcal{N}({0},{I})}\left[\left\|{\varepsilon}-{\varepsilon}_\theta(\bar{\alpha}_t{x}_0+\bar{\beta}_t{\varepsilon},y,t)\right\|^2\right]
		\label{1}
	\end{equation}
	\begin{equation}
		\mathbb{E}_{{x}_0\sim\tilde{p}({x}_0),{\varepsilon}\sim\mathcal{N}({0},{I})}\left[\left\|\varepsilon-{\varepsilon}_\theta(\bar{\alpha}_t{x}_0+\bar{\beta}_t{\varepsilon},t)\right\|^2\right]
		\label{2}
	\end{equation}
	
	$\varepsilon_\theta\left(x_t,y,t\right)=\varepsilon_1+\Delta_1$,  $\varepsilon_\theta\left(x_t,t\right)=\varepsilon_2+\Delta_2$. It is evident that $\varepsilon_1$ is the Gaussian noise added at a step of the forward process under the condition $y$, and $\varepsilon_2$ is the Gaussian noise added under the absence of condition $y$. These two distributions are independent and uncorrelated. Note that, $O=(1 + w)(\varepsilon_1 + \Delta_1) $-$ w(\varepsilon_2 + \Delta_2)$. Then, the form of its distribution can be depicted as Equation~\ref{3}. 
	
	\begin{equation}
		O\sim\left[(1+w)\Delta_1-w\Delta_2,((1+w)^2+w^2)I\right]
		\label{3}
	\end{equation}
	
	In the ideal situation, a well-trained neural network would not produce errors $\Delta_1=0, \Delta_2=0, \Delta=0$, then the expected output distribution is $O_L$. However, various sources of error may arise in the neural network in the sampling phase, i.e., the trained neural network itself has an error $c$, and the inaccuracy of the input data in the sampling phase causes the neural network to produce a fixed error $\Delta$, noting that the distribution is $O_G$. 
\end{proof}

Subsequently, we focus on how the variation in the parameter \( w \) affects the generation outcome, then focus on the discrepancy between $O_L$ and $O_G$ at each step. This analysis can assist us in evaluating the performance of the neural network model and in adjusting the parameter $w$ to minimize the error within the generation outcome. Then, we use the intersection area $J(w)$ between the distributions $O_L$ and $O_G$ as a more intuitive and illustrative measure of their similarity compared to the Kullback-Leibler~(KL) divergence.

\begin{lemma}
	Let \( d(w) = (1 + w) \Delta_1 - w \Delta_2 + \Delta = 2x_a \), where the sign of \( d(w) \) is consistent with the sign of \( x_a \). By applying the method for solving the integral probability density function of the Gaussian distribution, we derive the solution for the half intersection area \( I(w) \) between the distributions $O_L$ and $O_G$  as follows.
	\begin{align}
		I(w) =\frac{1}{2}J(w)= \frac{1}{2} e^{-\frac{d^2(w)}{8\sigma^2(w)}}
	 \label{function}
	\end{align}
	
    \label{lemma 4.1}
\end{lemma}
The proof of the following properties of \( I(w) \) is provided in Appendix~\ref{Lemma-3.1}.

Function \( I(w) \) quantifies the degree of overlap between the expected output distribution $O_L$ and the actual generated distribution $O_G$, an intuitive and concrete representation of their similarity, as depicted in Figure~\ref{fig.2}. An increase in the parameter \( w \) within a certain range enhances function \( I(w) \), indicating that the parameter \( w \) can improve the model's robustness to disturbances by aligning the actual generated distribution more closely with the ideal distribution, even in the presence of interference. As depicted in Figure~\ref{fig.2}, $x_a$ represents the $x$-coordinate of the intersection between distributions $O_L$ and $O_G$. It is determined by solving the probability density functions of the Gaussian distributions, yielding $x_a=\overline{x_{a}}=\frac{(1+w)\Delta_1-w\Delta_2+\Delta}{2}$.


\begin{figure}
	\centering
	\includegraphics[width=0.9\textwidth]{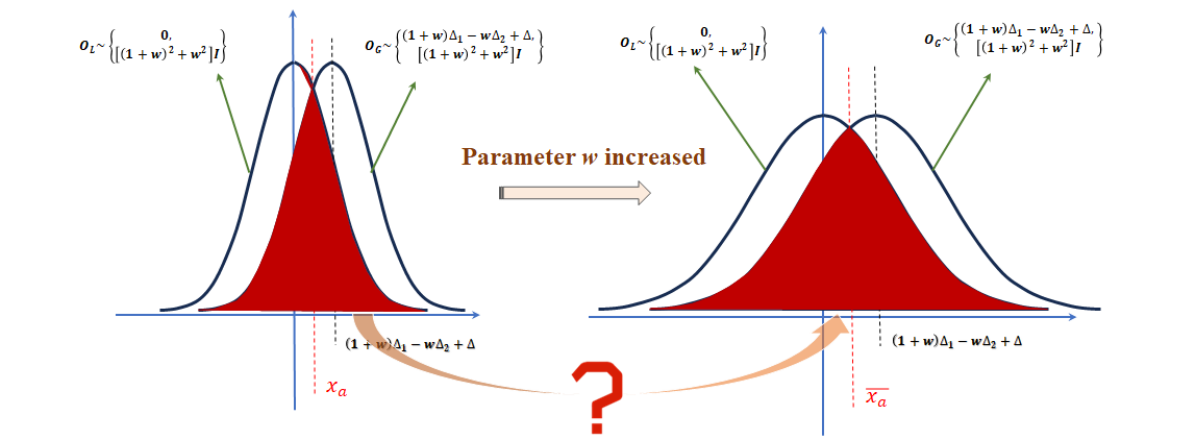}
	\caption{Discrepancy schematic distribution $O_L$ vs. $O_G$.} 
	\label{fig.2}
\end{figure}






\subsubsection{Monotonic Constraint}
Subsequently, we will explore the range of values for the $w$ for which the function $I(w)$ is monotonically increasing. Therefore, it will increase the intersection area between $O_L$ and $O_G$, then better align with the ideal scenario and enhance the ability to suppress interference as $w$ increases.

\begin{lemma}
	Let $\sigma(w) = \left[(1+w)^2 + w^2\right]^{\frac{1}{2}}$ be the standard deviation and $d(w) = (1+w)\Delta_1 - w\Delta_2 + \Delta$. Assume that the monotonic of the function $I(w)$ is independent of $\sigma(w)$, then the function $I(w)$ is monotonically increasing within the range $\left[0,\frac{\left|\Delta\right|}{2c}-\frac12\right]$.
	\label{lemma 4.2}
\end{lemma}

\begin{proof}
	This proof aims to analyze the monotonic of the function $I(w)$ concerning the parameter $w$. The specific analysis process is as follows.
	
	(1) Firstly, based on different values of $x_a$ and $2\Delta + \Delta_1 + \Delta_2$, the monotonic increasing intervals of the function $I(w)$ are determined.
	
	(2) In these monotonically increasing intervals, only some cases satisfy the prerequisite condition $w > 0$. Therefore, it is necessary to ensure that regardless of the variations in $\Delta$, $\Delta_1$ and $\Delta_2$, both the values of $x_a$ and $2\Delta + \Delta_1 + \Delta_2$ remain within these feasible ranges.
    
    The detail proof  is provided in Appendix~\ref{lemma-3.2}.
\end{proof}

\subsubsection{Function Value Constraint}

Within the intersection region $J(w)$ between distributions $O_L$ and $O_G$, the function is increasing within the range $\left[0,\frac{\left|\Delta\right|}{2c}-\frac12\right]$. However, beyond the value of $\frac{\left|\Delta\right|}{2c}-\frac12$, the function may also continue to increase monotonically with a certain probability. Therefore, $J(\frac{\left|\Delta\right|}{2c}-\frac12)$ is not necessarily the maximum value. Furthermore, whether $J(\frac{\left|\Delta\right|}{2c}-\frac12)$ satisfies the value function constraint $J(w)=1$ is a question that warrants in-depth consideration. If $J(\frac{\left|\Delta\right|}{2c}-\frac12)$ does not satisfy the value function constraint $J(w)=1$, then adjusting the parameter $w$ to meet this constraint may conflict with the monotonic constraint of the function being monotonically increasing over the interval $\left[0,\frac{\left|\Delta\right|}{2c}-\frac12\right]$.

\begin{lemma}
	Let the constant $n$ as the balancing factor, determined experimentally. When $w=\frac{n\left|\Delta\right|}{2c}-\frac12$, both monotonic and value constraints are generally satisfied, achieving equilibrium.
	\label{lemma 4.3}
\end{lemma}

\begin{proof}
	In analyzing the interrelationship between the monotonic constraint and the value constraint $J(w)=1$, the variable $d(w)$ plays an essential role. Taking $2\Delta + \Delta_1 + \Delta_2>0$ as an example, to ensure the monotonic increase of the function $J(w)$ for $w > 0$, $d(w)$ must remain non-negative. However, to satisfy the value constraint $J(w)=1$, $d(w)$ must be constantly zero. Although $d(w) = 0$ can simultaneously satisfy these two conditions, due to $\Delta_1$ and $\Delta_2$ being random variables, $d(w)$ itself possesses randomness and cannot be guaranteed to remain constant. Therefore, by analyzing the probability density function of $d(w)$ when $w=\frac{n\left|\Delta\right|}{2c}-\frac12$, we can observe how the balancing factor $n$ cleverly resolves the conflicts between these constraints. The detail proof  is provided in Appendix~\ref{lemma-3.3}.
\end{proof}

\section{Experiments and Discussion}
\label{experiments_discussion}
For the experiments detailed below, we employed two NVIDIA V100 GPUs with 32 gigabytes (GB) of memory each. To enhance reproducibility, we present our robust diffusion model framework through pseudocode. This standardized abstraction simplifies intricate operations, aiding comprehension and replication within the research community. The algorithmic workflow covers iterative training in the forward process of neural networks, as well as the design of filters and controllers in the reverse process. Please refer to Algorithms~\ref{alg-training-phase}$\sim$\ref{alg-sampling-phase} for details in Appendix~\ref{Experimental pseudocode}.
\subsection{Validation on MNIST Dataset }
\label{MNIST}
Exploring the determinacy of the balancing factor $n$ in the theory presents a challenge due to its elusive theoretical derivation. Our study adopts a semi-theoretical, semi-empirical approach, leveraging MNIST dataset to analyze the plausible range of $n$ values. The choice of MNIST dataset is twofold. Firstly, its relative simplicity allows us to focus on the impact of the balancing factor $n$. Secondly, conducting experiments under various systematic error conditions enhances the generalizability of the discovered balancing factor $n$ and provides a comprehensive assessment of its applicability.

To validate the generalizability of the balance factor $n$, we tested the performance of different parameter $w$ values under four distinct fixed errors $\Delta$~(0.3, -0.3, +3, -3). Through these experiments, we aimed to explore whether the range of values for the balance factor $n$ remains consistent across different error levels, thereby indirectly verifying the rationality of our approach. The experimental results are presented in Figures~\ref{fig.4}$\sim$~\ref{fig.5}, providing preliminary insights into the range of values for the balance factor $n$. Additional generation results are shown in Appendix~\ref{MNIST-}

\begin{figure}[H]
	\centering
	\begin{subfigure}[b]{0.2\textwidth}
		\includegraphics[width=\textwidth]{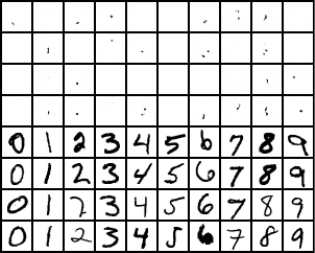}
		\caption{$w=0.5$}
		\label{fig:sub1}
	\end{subfigure}
	\hfill
	\begin{subfigure}[b]{0.2\textwidth}
		\includegraphics[width=\textwidth]{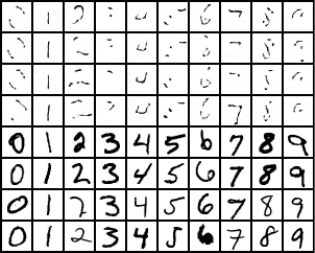}
		\caption{$w=5$}
		\label{fig:sub2}
	\end{subfigure}
	\hfill
	\begin{subfigure}[b]{0.2\textwidth}
		\includegraphics[width=\textwidth]{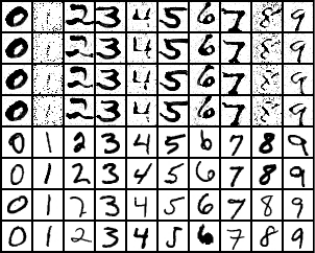}
		\caption{$w=50$}
		\label{fig:sub3}
	\end{subfigure}
	\hfill
	\begin{subfigure}[b]{0.2\textwidth}
		\includegraphics[width=\textwidth]{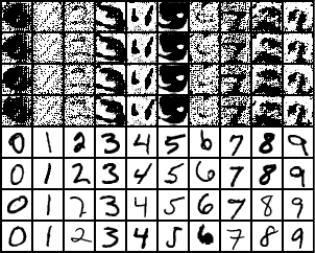}
		\caption{$w=500$}
		\label{fig:sub4}
	\end{subfigure}
	\caption{Reflections of \pmb{$\Delta= 0.3$} on different parameters $w$. With \pmb{$w=50$}, the model can generate numbers in the presence of errors.}
	\label{fig.4}
\end{figure}

\begin{figure}[H]
	\centering
	
	\begin{subfigure}[b]{0.2\textwidth}
		\includegraphics[width=\textwidth]{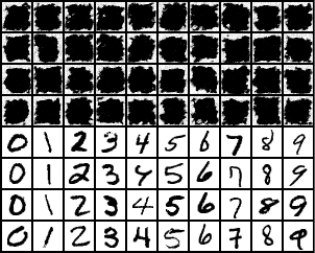}
		\caption{$w=0.5$}
		\label{fig:sub1}
	\end{subfigure}
	\hfill
	\begin{subfigure}[b]{0.2\textwidth}
		\includegraphics[width=\textwidth]{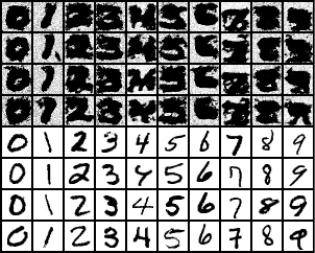}
		\caption{$w=5$}
		\label{fig:sub2}
	\end{subfigure}
	\hfill
	\begin{subfigure}[b]{0.2\textwidth}
		\includegraphics[width=\textwidth]{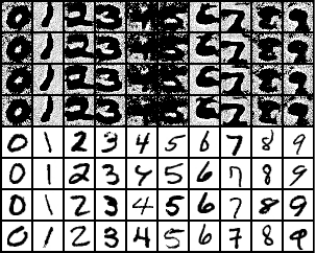}
		\caption{$w=50$}
		\label{fig:sub3}
	\end{subfigure}
	\hfill
	\begin{subfigure}[b]{0.2\textwidth}
		\includegraphics[width=\textwidth]{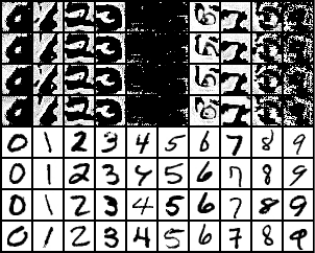}
		\caption{$w=500$}
		\label{fig:sub4}
	\end{subfigure}
	\caption{Reflections of \pmb{$\Delta= -0.3$} on different parameters $w$. With \pmb{$w=50$}, the model can generate numbers in the presence of errors.}
	\label{fig.5}
\end{figure}

In experiments on MNIST dataset , we studied how fixed error $\Delta$ affect parameter $w$ selection. We found that optimal model performance requires adjusting $w$ for different error levels. For fixed error $\Delta=\pm0.3$, $w=50$ was optimal, while $\Delta=3$ required $w=5000$. Analysis revealed lower image quality at $\Delta=-3$ compared to +3, possibly due to pixel normalization constraints, suggesting a model robustness threshold. Despite these variations, an optimal $n$ of around 185 was inferred from comparisons across error conditions. To maintain model robustness and image quality, we suggest adjusting $w$ based on error magnitude while keeping $n$ near 185. This strategy ensures adaptability to diverse error scenarios, supporting high-quality image generation for various applications.

\subsection{Validation on CIFAR-10 Dataset}
\label{CIFAR-10}

To ascertain the suitability of n=185 beyond MNIST dataset, we extended our experiments to the more complex CIFAR-10 dataset. The increased complexity of CIFAR-10 offers a stringent test for our model. We employed the Fréchet Inception Distance (FID) to mitigate the subjectivity inherent in visual assessment. FID evaluates the statistical similarity between the feature representations of generated and actual images, with a lower score indicating a closer distribution match and superior image generation quality.
Furthermore, we modified the neural network architecture and sampling mode alongside the dataset to comprehensively test the robustness of the balance factor $n$. This approach ensures that our conclusions are not confined to a specific neural network and sampling mode configuration, enhancing the generalizability of our findings.

The FID assessments across varied neural network architectures on CIFAR-10 dataset reinforce and expand upon the conclusions drawn from MNIST. This cross-architecture and cross-dataset validation strengthens the scientific rigor of our research and offers a more profound understanding of the model's performance across different conditions.

\begin{table}[h]
	\captionsetup{justification=raggedright,singlelinecheck=false} 
	\caption{Generated outcomes of $\Delta$ on $w$ across network structures and sampling modes. The optimal $w$ is set between the two bolded $w$.}
	\centering
	\begin{tabular}{c >{\centering\arraybackslash}m{0.7in} >{\centering\arraybackslash}m{0.7in} >{\centering\arraybackslash}m{0.8in} >{\centering\arraybackslash}m{0.8in}}
		\specialrule{1.2pt}{1pt}{1pt}
		\multirow{2}{*}{Sampling mode} & \multirow{2}{*}{$\Delta$} & \multirow{2}{*}{$w$} & \multicolumn{2}{c}{Network structure}  \\ 
		\cmidrule{4-5}
		& & & U-net & U-vit \\ 
		\cmidrule[\heavyrulewidth]{1-1}\cmidrule[\heavyrulewidth]{2-5}
		DDIM & 0 & 0 & \textbf{3.572} & \textbf{4.109} \\ 
		\cmidrule{1-5}
		\multirow{7}{*}{DDPM} & \multirow{4}{*}{0.03} & 0 & 96.745 & 97.630 \\
		& & 5 & 42.278 & 48.867 \\
		& & 10 & \textbf{37.934} & 41.627 \\
		& & 20 & \textbf{43.512} & \textbf{37.980} \\ 
		\cmidrule{2-5}
		& \multirow{3}{*}{0.3} & 0 & 543.202 & 609.914 \\
		& & 50 & \textbf{90.171} & \textbf{77.047} \\
		& & 100 & \textbf{105.414} & \textbf{87.249} \\ 
		\cmidrule{1-5}
		\multirow{7}{*}{DDIM} & \multirow{4}{*}{-0.03} & 0 & 62.066 & 66.079 \\
		& & 5 & \textbf{38.828} & 37.902 \\
		& & 10 & \textbf{41.387} & \textbf{35.453} \\
		& & 20 & 46.459 & \textbf{35.720} \\ 
		\cmidrule{2-5}
		& \multirow{3}{*}{-0.3} & 0 & 416.859 & 424.844 \\
		& & 50 & \textbf{83.808} & \textbf{66.923} \\
		& & 100 & \textbf{90.447} & \textbf{69.050} \\
		\specialrule{1.2pt}{1pt}{1pt}
	\end{tabular}
\end{table}

}

\begin{figure}[htbp]
	\centering
	\includegraphics[width=0.95\textwidth]{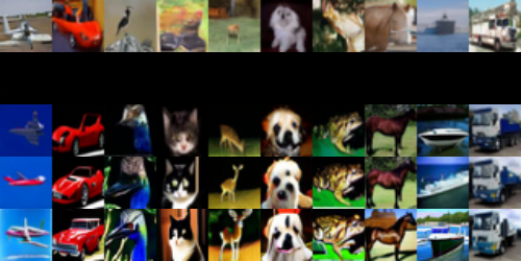}
	\caption{Visualization of the final generation results in \pmb{$\Delta=0.03$}, \textbf{DDPM} and \textbf{U-net} on CIFAR-10 dataset. The top row of the image illustrates the ideal state where the CDM generates images without any error. Subsequent 2$\sim$5 rows, demonstrate the images produced by RCDM with the \( w \) set to \textbf{0, 5, 10,} and \textbf{20}, respectively.} 
	\label{fig.21}
	\vspace{-10pt}
\end{figure}

In this study, we systematically investigated the impact of $w$ on the generative outcomes under various fixed error \( \Delta \) through a series of controlled experiments. The findings indicated a notable decrease in the influence of errors on the final generation results within a particular range of increasing  $w$, as evidenced by the decline in the FID score. Specifically, as \( w \) increased, the FID score generally showed a downward trend, indicating a reduction in the distributional difference between the generated and authentic images, thereby validating our hypothesis regarding the impact of \( w \) adjustment on generation quality. Furthermore, we examined the changes in the FID index under different network structures and sampling modes, further revealing the adaptability and effectiveness of \( w \) adjustment under various conditions. Experiments conducted on MNIST dataset indicated that setting the balance factor to 185 could significantly enhance the robustness of the diffusion model. Cross-validation revealed that the parameter \( w \) setting corresponding to this balancing factor could notably improve the model's resistance to error on CIFAR-10 dataset.
\begin{figure}[!ht]
	\centering
	\includegraphics[width=0.95\textwidth]{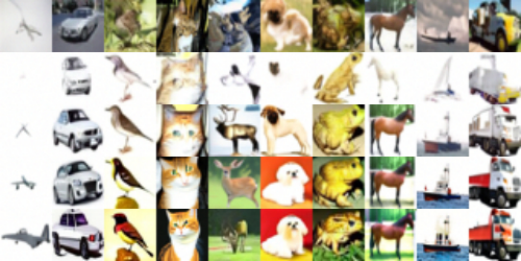}
	\caption{Visualization of the final generation results in \pmb{$\Delta=-0.03$}, \textbf{DDIM} and \textbf{U-vit} on CIFAR-10 dataset. The top row of the image illustrates the ideal state where CDM generates images without any error. Subsequent 2$\sim$5 rows, demonstrate the images produced by RCDM with the \( w \) set to 0, 5, 10, and 20, respectively.} 
	\label{fig.22}
\end{figure}

However, when we further increased $ w $, for instance, by adjusting the balance factor to 120, we observed a slight enhancement in the model's anti-interference capability. This phenomenon differs from the experimental results of MNIST dataset, which we attribute primarily to MNIST dataset's $ w $ being determined based on perceptual judgment and the sparse parameter settings in the experiments. Notably, in the experiments on CIFAR-10 dataset, when the neural network structure was U-vit and the sampling method was DDIM, the result with $ w $ set to 20 outperformed that with $ w $ set to 185, contrasting with other experimental outcomes. We speculate that this occurrence may stem from two reasons: firstly, the formula for setting the optimal $ w $ is a semi-empirical formula that combines theoretical analysis with practical experience and may have certain inaccuracies. Secondly, the range of optimal $ w $ determined experimentally is relatively limited.

Our findings suggest that setting the balance factor to 185 can provide sufficient robustness for the diffusion model in most cases. Pursuing a larger balance factor may result in only a slight improvement in robustness, which may not be cost-effective. In addition, the comparison between experimental results with and without error input demonstrates that although $ w $ can enhance the model's robustness and mitigate the impact of errors to some extent, it cannot eliminate errors. Therefore, reducing error input remains critical for improving the quality of generation outcomes.

These discoveries not only provide a solid theoretical foundation for parameter adjustment in generative models but also, have significant practical guidance for error control and quality optimization. Through these experiments, we have provided a systematic methodology for parameter optimization in generative models, which will guide future research and application development.

\section{Conclusion}

This study presents the Robust Conditional Diffusion Model (RCDM), which enhances the robustness of diffusion models by considering collaborative interaction between two neural networks based on control theory. RCDM optimizes the weighting of two neural networks during sampling without additional computational costs, maintaining efficiency and adaptability. Experimental results on MNIST and CIFAR-10 datasets demonstrate significant improvements in robustness without compromising inference speed, making it practical for real-time applications. Our proposed model is a trade-off strategy that reduces the diversity of model outputs to some extent but still achieves acceptable results in extreme cases. Overall, this work strengthens diffusion model robustness, advancing neural network-based machine learning systems and enabling more reliable AI solutions.


\bibliographystyle{plain}
\bibliography{reference}


\appendix
\renewcommand{\thefigure}{A.\arabic{figure}}
\setcounter{figure}{0}
\section{Appendix}

\subsection{Conditional Diffusion Model}
\label{Diffusion model}

Conditional diffusion models are identical to diffusion models in that they have two processes, i.g. forward process and reverse process. The mathematical mechanism of conditional diffusion models can be summarized as follows.

\textbf{Forward Process.} During the forward diffusion phase, the model gradually adds Gaussian noise to the data through a series of timesteps $t$. At each step, the data $x_{t-1}$ is perturbed by Gaussian noise with variance $\beta_t$ to generate a new latent variable $x_{t}$. 
\begin{equation}
	q(x_t|x_{t-1})=N(x_t;\sqrt{1-\beta_t}x_{t-1},\beta_t\mathbf{I}),
	\label{0.1}
\end{equation}
\textbf{Reverse Process.} After numerous steps in forward diffusion, the data $x_{T}$ approaches a Gaussian distribution. The model must learn to recover the original data from the noise through the reverse process, typically achieved by training a neural network to predict the mean $\boldsymbol{\mu}_\theta(\mathbf{x}_t,t)$ and variance $\boldsymbol{\Sigma}_\theta(\mathbf{x}_t,t)$ of each step.
\begin{equation}
	p_\theta(\mathbf{x}_{t-1}|\mathbf{x}_t)=\mathcal{N}(\mathbf{x}_{t-1};\boldsymbol{\mu}_\theta(\mathbf{x}_t,t),\boldsymbol{\Sigma}_\theta(\mathbf{x}_t,t)),
	\label{0.2}
\end{equation}
\textbf{Conditional Guidance.} Conditional Diffusion Models guide the generation process by incorporating additional conditional information (e.g., class labels or text embeddings). The incorporation of the condition $y$ into each step of the diffusion process is what enables the generation mechanism to be contingent upon this conditioning data.
\begin{equation}
	p_\theta(\mathbf{x}_{0:T}|y)=p_\theta(\mathbf{x}_T)\prod_{t=1}^Tp_\theta(\mathbf{x}_{t-1}|\mathbf{x}_t,y),
	\label{0.3}
\end{equation}
\textbf{Joint Training.} In~\cite{ho2022classifier}, the conditional diffusion model is trained jointly with the unconditional model, allowing the model to learn how to generate both conditional and unconditional samples.
\begin{equation}
	\tilde{\varepsilon}_\theta(x_t,y,t)=(1+w)\varepsilon_\theta(x_t,y,t)-w\varepsilon_\theta(x_t,t),
	\label{0.4}
\end{equation}

\subsection{Experimental Pseudocode}
\label{Experimental pseudocode}
In this section, we provided pseudocode for implementing RCDM to facilitate a more direct comprehension of the algorithm's logical sequence for the reader. RCDM shares structural similarities with CDM, both comprising a forward and a reverse process. During the forward process, training involves both a non-conditional neural network and a conditional neural network. The core advantage of RCDM lies in its reverse process, where the weight $w$ of both networks can be dynamically adjusted based on the magnitude of the error, thereby enhancing the model's robustness. Consequently, RCDM, through the dynamic adjustment of $w$ during the reverse process, strengthens its capability to generate high-quality images in environments with interference.

\begin{figure}[H]
	\centering
	
	\begin{minipage}[t]{0.45\textwidth}
		\begin{algorithm}[H]
			
			\caption{Training phase}
			\label{alg-training-phase}
			\begin{algorithmic}[1]
				\setstretch{1.25}
				\State \textbf{repeat}
				\State $\mathbf{x}_0\sim q(\mathbf{x}_0)$
				\State $t\sim\operatorname{Uniform}(\{1,\ldots,T\})$
				\State $\boldsymbol{\varepsilon}\sim\mathcal{N}(\mathbf{0},\mathbf{I})$
				\State Stepwise gradient descent
				\Statex
				\quad  Unconditional neural network:
				\Statex
				\quad \quad $\nabla_\theta\left\|\boldsymbol{\varepsilon}-\boldsymbol{\varepsilon}_\theta(\sqrt{\bar{\alpha}_t}\mathbf{x}_0+\sqrt{1-\bar{\alpha}_t}\boldsymbol{\varepsilon},t)\right\|^2$
				\Statex \quad	Conditional neural network
				\Statex \quad \quad 
				$\nabla_\theta\left\|\boldsymbol{\varepsilon}-\boldsymbol{\varepsilon}_\theta(\sqrt{\bar{\alpha}_t}\mathbf{x}_0+\sqrt{1-\bar{\alpha}_t}\boldsymbol{\varepsilon},y,t)\right\|^2$
				\State \textbf{until} converged
			\end{algorithmic}
		\end{algorithm}
	\end{minipage}
	\hfill 
	\begin{minipage}[t]{0.45\textwidth}
		\begin{algorithm}[H]
			
			\caption{Sampling phase}
			\label{alg-sampling-phase}
			\begin{algorithmic}[1]
				\setstretch{1.25}
				\State $\mathbf{x}_T\sim\mathcal{N}(\mathbf{0},\mathbf{I})$
				\State $\textbf{for }t=T,\ldots,1\textbf{ do}$
				\State $\mathbf{z}\sim\mathcal{N}(\mathbf{0},\mathbf{I})\mathrm{~if~}t>1,\text{else }\mathbf{z}=\mathbf{0}$
				\State Filter math formula
				\Statex $\quad\quad \quad \quad \quad w=\frac{\mathrm{n}|\Delta|}{2c}-\frac12$
				\State Controller signal 
				\Statex ${\tilde{\varepsilon}_\theta=\left[(1+w)\varepsilon_{\theta}(x_{t},y,t)-w\varepsilon_{\theta}(x_{t},t)\right]}$
				\State Generator signal 
				\Statex $\mathbf{x}_{t-1}=\frac1{\sqrt{\alpha_t}}\left(\mathbf{x}_t-\frac{1-\alpha_t}{\sqrt{1-\bar{\alpha}_t}}\tilde{\varepsilon}_\theta\right)+\sigma_t\mathbf{z}$ 
				\State $\textbf{end for}$
				\State $\textbf{return x}_0$
			\end{algorithmic}
		\end{algorithm}
	\end{minipage}
\end{figure}

\subsection{Proofs}
In this section, we provided the proofs for our main theoretical results (Theorem~\ref{theorem 3.1} and Lemmas~\,\ref{lemma 4.1}~$\sim$~\ref{lemma 4.3}). 

\subsubsection{Proof for Theorem \ref{theorem 3.1}}
\label{Proof for Theorem 3.1}
We utilized mathematical induction to examine the accumulation of the final error $\Delta$.
\begin{equation}
	\begin{aligned}
		&X_{T}:0 \\
		&X_{T-1}:-\frac{\beta_T^2}{\sqrt{\alpha_T}}\frac{\beta_T}{\beta_T}\Delta \\
		&X_{T-2}:-\Bigg(\frac{\beta_T^2}{\sqrt{\alpha_{T-1}\alpha_T}\beta_T}+\frac{\beta_{T-1}^2}{\sqrt{\alpha_{T-1}}\overline{\beta_{T-1}}}\Bigg)\Delta \\
		&\cdots\cdots \\
		&X_0:-\Bigg(\frac{\beta_T^2}{\sqrt{\overline{\alpha}_T}\overline{\beta}_T}+\frac{\beta_{T-1}^2}{\sqrt{\overline{\alpha}_{T-1}}\overline{\beta}_{T-1}}+\cdots\frac{\beta_1^2}{\sqrt{\alpha_1}\beta_1}\Bigg)\Delta 
		\label{1.5}
	\end{aligned}
\end{equation}
Here, $\beta_T=\sqrt{1-\alpha_T} ,\quad\overline{\beta}_T=\sqrt{1-\overline{\alpha}_T}$. We focused on the maximum term $-\frac{\beta_T^2}{\sqrt{\overline{\alpha}_T}\overline{\beta_T}}\Delta=-\frac{1-a_T}{\sqrt{\overline{\alpha}_T(1-\overline{\alpha}_T)}}\Delta $. When $T$ is large and $\overline{\alpha}_T$ is very small, any inaccuracies in the input data that result in systematic errors in the neural network's output will cause the maximum term to become exceedingly large. Thus, the errors arising from the neural network's response to input inaccuracies will undergo substantial amplification through coefficient multiplication, ultimately leading to a loss of control over the final generation effect.

\subsubsection{Proof for Lemma~\ref{lemma 4.1}}
\label{Lemma-3.1}

Due to the symmetry of the area where $O_G$ intersects $O_L$, half the size of its area reflects the variation of the intersecting area with the parameter $ w $, as shown in Figure~\ref{Description of the area function}

\begin{figure}[h]
	\centering
	\includegraphics[width=0.9\textwidth]{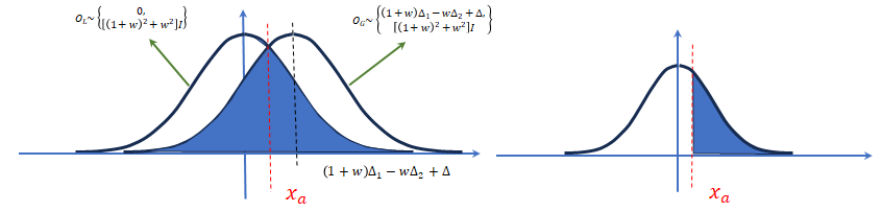}
	\caption{Description of the area function $J(w)$.} 
	\label{Description of the area function}
\end{figure}

According to the nature of Gaussian distribution, the expression of this area function as follows.

(1) $x_a\geq0$

\begin{equation}
	I(w)=\int_{\frac{d(w)}2}^{+\infty}\frac1{\sqrt{2\pi}\sigma(w)}e^{-\frac{x^2}{2\sigma(w)^2}}dx
	\label{18}
\end{equation}

(2) $x_a<0$

\begin{equation}
	I(w)=\int_{-\infty}^{\frac{-d(w)}2}\frac1{\sqrt{2\pi}\sigma(w)}e^{-\frac{x^2}{2\sigma(w)^2}}dx
	\label{19}
\end{equation}

We refered to the method of solving the integral probability density function of the Gaussian distribution.
For $x_a\geq0$, let $z=\frac x\sigma$, then $dx=\sigma(w)dz$. Given that $\frac d2\leq x\leq+\infty$, it follows that $\frac d{2\sigma}\leq z\leq+\infty$. We assumed the standard deviation $\sigma>0$. The case for $\sigma<0$ yields identical results in the subsequent derivation of the proof, making the discussion of the sign of $\sigma$ moot. The condition $\sigma\neq0$ is sufficient for our analysis.
\begin{equation}
	I(w)=\int_{\frac{d(w)}{2\sigma(w)}}^{+\infty}\frac{1}{\sqrt{2\pi}}e^{-\frac{z^2}{2}}dz
	\label{20}
\end{equation}

\begin{equation}
	I^2(w)=\int_{\frac{d(w)}{2\sigma(w)}}^{+\infty}\frac{1}{\sqrt{2\pi}}e^{-\frac{a^2}{2}}da\int_{\frac{d(w)}{2\sigma(w)}}^{+\infty}\frac{1}{\sqrt{2\pi}}e^{-\frac{b^2}{2}}db=\frac{1}{2\pi}\int_{\frac{d(w)}{2\sigma(w)}}^{+\infty}\int_{\frac{d(w)}{2\sigma(w)}}^{+\infty}e^{-\frac{(a^2+b^2)}{2}}dadb
	\label{21}
\end{equation}

Let $a=rcos\theta,b=rsin\theta$ in polar coordinates, then $dadb=rdrd\theta$.

\begin{align}
	I^{2}(w)
	& = \frac{1}{2\pi}\int_{0}^{\frac{\pi}{2}}d\theta\int_{\frac{\sqrt{2}d(w)}{2\sigma(w)}}^{+\infty}e^{-\frac{r^{2}}{2}}rdr \notag\\
	&=\quad\frac14\int_{\frac{\sqrt{2}d(w)}{2\sigma(w)}}^{+\infty}e^{-\frac{r^2}2}d\frac{r^2}{2} \notag\\
	&=\quad\frac14\int_{\frac{d^2(w)}{4\sigma^2(w)}}^{+\infty}e^{-c}dc \notag\\
	&=\quad\frac14(-e^{-c})|_{\frac{d^2(w)}{4\sigma^2(w)}}^{+\infty} \notag\\
	&=\begin{array}{c}\frac14e^{-\frac{d^2(w)}{4\sigma^2(w)}}\end{array}
	\label{22}
\end{align}

When $x_a\geq0$, the function $I(w) = \frac12e^{-\frac{d^2(w)}{8\sigma^2(w)}}$. Similarly, when $x_a<0$, the function $I(w) = \frac12e^{-\frac{d^2(w)}{8\sigma^2(w)}}$. In conclusion, regardless of the value of $x_a$, the function $I(w) =\frac12e^{-\frac{d^2(w)}{8\sigma^2(w)}}$.

\subsubsection{Proof for Lemma~\ref{lemma 4.2}}
\label{lemma-3.2}
Obviously, the function $J(w)$ agrees with the monotonic of the function $I(w)$ as in Equation \ref{function}.

(1) When $x_a>0$, that is $d(w)>0$. 

\begin{align}
	I^{\prime}(w)=\frac12e^{-\frac{d^2(w)}{8\sigma^2(w)}}\cdot(-\frac{d^2(w)}{8\sigma^2(w)})^{\prime}
	&\geq\quad0
	\notag\\
	(-\frac{d^2(w)}{8\sigma^2(w)})^{\prime}&\geq\quad0
	\notag\\
	-\frac{16d^{\prime}(w)d(w)\sigma^2(w)-16\sigma^{\prime}(w){\sigma(w)}d^2(w)}{64\sigma^{4}(w)}&\geq\quad0
	\notag\\
	-d^{\prime}(w)d(w)\sigma^2(w)+\sigma^{\prime}(w){\sigma(w)}d^2(w)&\geq\quad0
	\notag\\
	-d^{\prime}(w)\sigma(w)+\sigma^{\prime}(w)d(w)&\geq\quad0
	\label{24}
\end{align}

Where $\sigma(w)=[(1+w)^{2}+w^{2}]^{\frac{1}{2}}$,  $d(w)=\Delta_{1}+\Delta+w(\Delta_{1}-\Delta_{2})$. Let $D=(1+w)^{2}+w^{2}$, $f=\Delta_{1}+\Delta$, $C=\Delta_1-\Delta_2$. Then,

\begin{align}
	D^{-\frac{1}{2}}\left(2w+1\right)[\Delta_{1}+\Delta+wC]-CD^{\frac{1}{2}}& \geq0 \notag\\
	[2w+1](f+wC)-C[2w^{2}+2w+1]& \geq0 \notag\\
	[2f-C]w+f-C& \geq0 \notag\\
	\begin{aligned}[2(\Delta_1+\Delta)-(\Delta_1-\Delta_2)]w+(\Delta+\Delta_2)\end{aligned}& \geq0 \notag\\
	(2\Delta+\Delta_2+\Delta_1)w+\Delta+\Delta_2& \geq0 \notag\\
	(2\Delta+\Delta_2+\Delta_1)w& \geq-(\Delta+\Delta_2) 
	\label{25}
\end{align}

Thus,

if $2\Delta+\Delta_1+\Delta_2>0$, the function $I(w)$ is monotonically increasing when $w\geq -\frac{\Delta+\Delta_2}{2\Delta+\Delta_2+\Delta_1}$;

if $2\Delta+\Delta_1+\Delta_2<0$, the function $I(w)$ is monotonically increasing when $w\leq -\frac{\Delta+\Delta_2}{2\Delta+\Delta_2+\Delta_1}$;

if $2\Delta+\Delta_1+\Delta_2=0$, the function $I(w)$ is monotonically increasing regardless of the value of $w$.

(2) When $x_a<0$, that is $d(w)<0$.

Similarly, we can obtain,

if $2\Delta+\Delta_1+\Delta_2>0$, the function $I(w)$ is monotonically increasing when $w\leq -\frac{\Delta+\Delta_2}{2\Delta+\Delta_2+\Delta_1}$;

if $2\Delta+\Delta_1+\Delta_2<0$, the function $I(w)$ is monotonically increasing when $w\geq -\frac{\Delta+\Delta_2}{2\Delta+\Delta_2+\Delta_1}$;

if $2\Delta+\Delta_1+\Delta_2=0$, the function $I(w)$ is monotonically increasing regardless of the value of the parameter $w$.

Next focus on $-\frac{\Delta+\Delta_2}{2\Delta+\Delta_2+\Delta_1}$, let $\Delta_1-\Delta_2=h(\Delta+\Delta_2)$. 

\begin{align}
	-\frac{\Delta + \Delta_2}{2\Delta + \Delta_2 + \Delta_1} &= -\frac{1}{2 + \frac{b}{\Delta + \Delta_2}} \notag \\
	&= -\frac{1}{2 + \frac{h(\Delta + \Delta_2)}{\Delta + \Delta_2}} \notag \\
	&= -\frac{1}{2 + h}
	\label{26}
\end{align}

Neural network in the training process $\varepsilon_\theta\left(x_t,y,t\right),\varepsilon_\theta\left(x_t,t\right)$ using the same set of network structure, and in the code practice, both of them are the same super parameter. Thus $\Delta_1- \Delta_2$ is small, and then $h$ is the small value whose positive or negative cannot be determined. But the $n$ value, being very small, does not affect the overall $-\frac{\Delta+\Delta_2}{2\Delta+\Delta_2+\Delta_1}$ positively or negatively. 

\begin{equation}
	-\frac{\Delta+\Delta_2}{2\Delta+\Delta_2+\Delta_1} = -\frac{1}{2+n} <0
	\label{27}
\end{equation}

When $2\Delta+\Delta_1+\Delta_2>0$ and $d>0$, $w>0$ increasing can improve the robustness of the conditional diffusion model. However, if increasing $2\Delta+\Delta_1+\Delta_2<0$ and $d>0$, $w>0$ increasing does not improve the robustness of the conditional diffusion model. The analysis of all the situations are summarized in Figure ~\ref{图8}.

\begin{figure}[!htbp]
	\centering
	\includegraphics[width =0.8\textwidth]{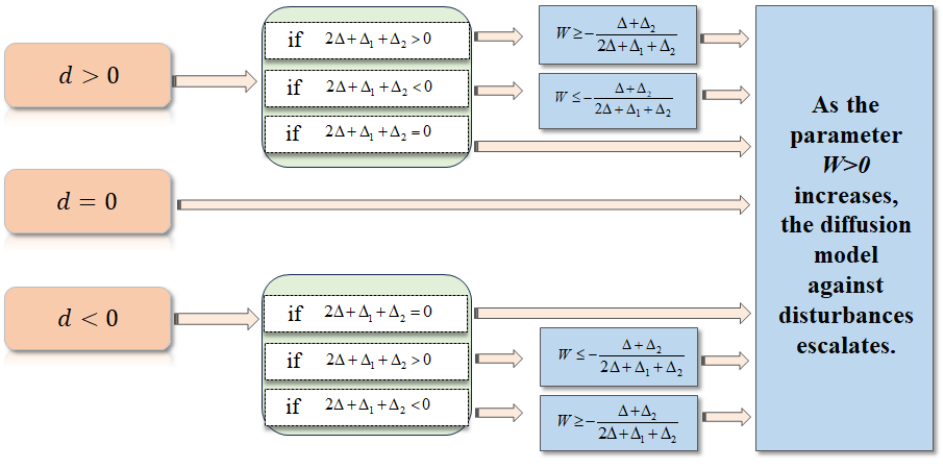}
	\caption{Situation analysis}
	\label{图8}
\end{figure}

In simple terms, when $\Delta$ is sufficiently large or small, the $2\Delta+\Delta_1+\Delta_2$ matches that of $\Delta$. It leads to the conclusion that if $\Delta>0$, as long as $d\geq0$, increasing the parameter $w$ enhances the robustness of the diffusion model against interference. Similarly, if $\Delta<0$, as long as $d\leq 0$, increasing the parameter $w$ improves the model's resistance to interference. Therefore, we only need to consider the $d$ in different scenarios with varying $\Delta$. Where $\Delta_2$ and $\Delta_1$ are variables that vary within the range of $[-c,c]$.

(1)$\Delta>0$
\begin{align}
	d=(1+w)\Delta_1-w\Delta_2+\Delta\geq0 \notag\\
	[(1+w)\Delta_1-w\Delta_2+\Delta]_{min}\geq0\notag\\
	[-c(1+2w)+\Delta]\geq0\notag\\
	w\leq\frac{\Delta}{2c}-\frac{1}{2}
	\label{28}
\end{align}
When $\Delta>0$ and is sufficiently large, within the range of $[0,\frac{\Delta}{2c}-\frac{1}{2}]$, the function $I(w)$ monotonically increases as the parameter $w$ increases.

(2)$\Delta<0$
\begin{align}
	d=(1+w)\Delta_1-w\Delta_2+\Delta\leq0 \notag\\
	[(1+w)\Delta_1-w\Delta_2+\Delta]_{max}\leq0\notag\\
	[c(1+2w)+\Delta]\leq0\notag\\
	w\leq-\frac{\Delta}{2c}-\frac{1}{2}
	\label{29}
\end{align}
When $\Delta<0$ and is sufficiently large, within the range of $[0,-\frac{\Delta}{2c}-\frac{1}{2}]$, the function $I(w)$ monotonically increases as the $w$ increases.

In summary, the function $I(w)$ must be monotonically increasing in the range $\left[0,\frac{\left|\Delta\right|}{2c}-\frac12\right]$.

\subsubsection{Proof for Lemma \ref{lemma 4.3}}
\label{lemma-3.3}

Let the random variable $a=(1+w)\Delta_{1}$ with $a\in[-(1+w)c,(1+w)c]$, and $b=w\Delta_{2}$ with $b\in[-wc,wc]$, where $\Delta_1$ and $\Delta_2$ are independent variables within the range of $[-c,c]$. Then the probability density function $p(d)$ of $d(w)=(1+w)\Delta_1-w\Delta_2+\Delta$ can be solved as follows.
\begin{equation}
	p(d) = 
	\begin{cases}
		\frac{2cw + c + d}{4c^2w(w + 1)}\, & \text{if } -2cw - c + \Delta \leq d \leq -c + \Delta, \vspace{0.5em} \\
		\frac{1}{2c(w + 1)}\, & \text{if } -c + \Delta \leq d \leq c + \Delta, \vspace{0.5em} \\
		\frac{2cw + c - d}{4c^2w(w + 1)}\, & \text{if } c + \Delta \leq d \leq 2cw + c + \Delta.
	\end{cases}
	\label{28}
\end{equation}
Under different values of $\Delta$, the probability density function $p(d | w)$ exhibits sensitivity to $w$. Specifically, when $w$ takes the values of $\frac{n\left|\Delta\right|}{2c}-\frac12$ and $\frac{\Delta}{2c}-\frac12$, respectively, we have illustrated the trend of $p(d | w)$ through graphical representations. The specific changes of $p(d | w)$ are shown in Figure \ref{图9}.

\begin{figure}[!ht]
	\centering
	\includegraphics[width =0.8\textwidth]{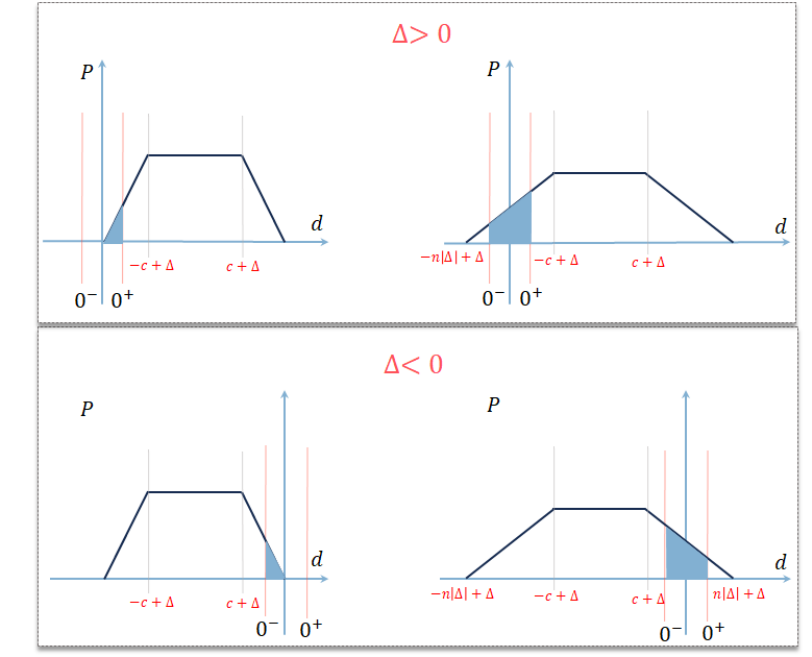}
	\caption{Schematic of balance factor n role}
	\label{图9}
\end{figure}

The increase of the balance factor $n$ significantly affects the function $d(w)$. When $\Delta>0$, the variable $d(w)$ still has a high probability of satisfying $d(w)\geq 0$ as $n$ increases. When $\Delta<0$, $d(w)$ still has a high probability of satisfying $d(w)\leq 0$. Additionally, increasing  $n$ allows $d(w)$ to maintain a certain probability near zero. This adjustment effectively balances the fulfillment of function value constraints and the maintenance of monotonic.

\subsection{Experimental Results}
\subsubsection{Extended Results on MNIST Dataset}
\label{MNIST-}

Herein, we present the experimental outcomes from Section \ref{MNIST} in detail as a supplement to the results discussed in the main paper. Figures \ref{fig.6}$\sim$\ref{fig.7} illustrate the results and provide visual evidence of the phenomena observed during the experiments.

\begin{figure}[H]
	\centering
	
	\begin{subfigure}[b]{0.2\textwidth}
		\includegraphics[width=\textwidth]{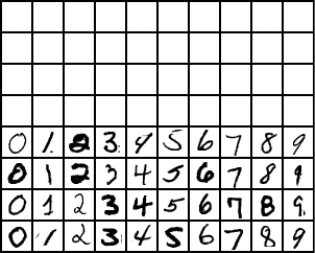}
		\caption{$w=5$}
		\label{fig:sub1}
	\end{subfigure}
	\hfill
	\begin{subfigure}[b]{0.2\textwidth}
		\includegraphics[width=\textwidth]{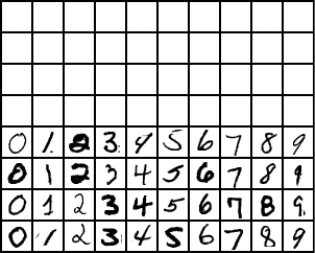}
		\caption{$w=50$}
		\label{fig:sub2}
	\end{subfigure}
	\hfill
	\begin{subfigure}[b]{0.2\textwidth}
		\includegraphics[width=\textwidth]{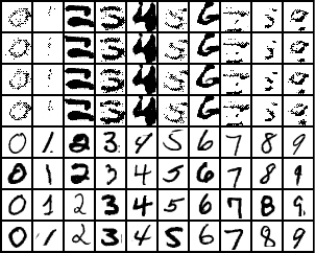}
		\caption{$w=500$}
		\label{fig:sub3}
	\end{subfigure}
	\hfill
	\begin{subfigure}[b]{0.2\textwidth}
		\includegraphics[width=\textwidth]{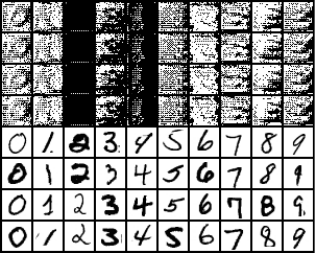}
		\caption{$w=5000$}
		\label{fig:sub4}
	\end{subfigure}
	\caption{Reflections of \pmb{$\Delta= 3$} on different valves of $w$. With \pmb{$w=500$}, the model can generate numbers in the presence of errors.}
	\label{fig.6}
\end{figure}

\begin{figure}[H]
	\centering
	
	\begin{subfigure}[b]{0.2\textwidth}
		\includegraphics[width=\textwidth]{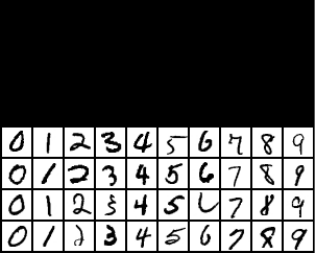}
		\caption{$w=5$}
		\label{fig:sub1}
	\end{subfigure}
	\hfill
	\begin{subfigure}[b]{0.2\textwidth}
		\includegraphics[width=\textwidth]{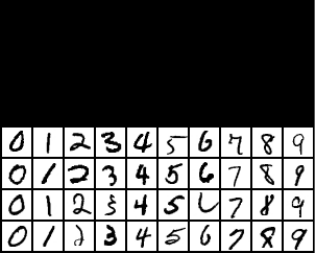}
		\caption{$w=50$}
		\label{fig:sub2}
	\end{subfigure}
	\hfill
	\begin{subfigure}[b]{0.2\textwidth}
		\includegraphics[width=\textwidth]{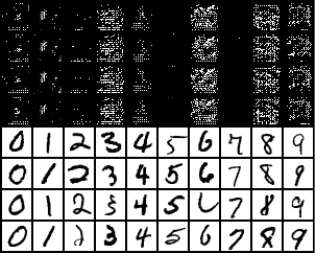}
		\caption{$w=500$}
		\label{fig:sub3}
	\end{subfigure}
	\hfill
	\begin{subfigure}[b]{0.2\textwidth}
		\includegraphics[width=\textwidth]{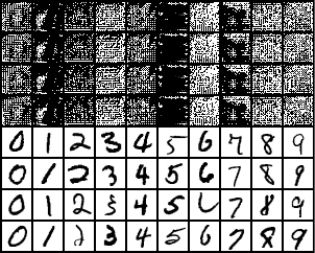}
		\caption{$w=5000$}
		\label{fig:sub4}
	\end{subfigure}
	\caption{Reflections of \pmb{$\Delta= -3$} on different valves of $w$.}
	\label{fig.7}
\end{figure}


\subsubsection{Extended Results on CIFAR-10 Dataset}
To ensure the comprehensiveness and transparency of this paper, following the detailed descriptions of our experimental design, procedures, and key findings, we provide supplementary graphical materials in this section. These materials include comparative visualizations across various network structures, the effectiveness of different sampling methods, and the impact of varying $w$ values under different error conditions. 
\begin{figure}[!ht]
	\centering
	\includegraphics[width=1.0\textwidth]{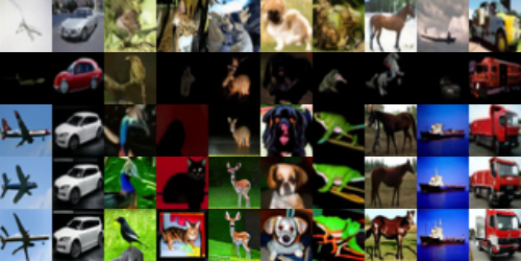}
	\caption{Visualization of the final generation results in \pmb{$\Delta=0.03$}, \textbf{DDPM} and \textbf{U-vit} on CIFAR-10 dataset.The top row of the image illustrates the ideal state where the CDM generates images without any error. Subsequent 2$\sim$5 rows, demonstrate the images produced by RCDM with the  \( w \) set to 0, 5, 10, and 20, respectively.} 
	\label{fig.23}
\end{figure}

\begin{figure}[H]
	\centering
	\includegraphics[width=1.0\textwidth]{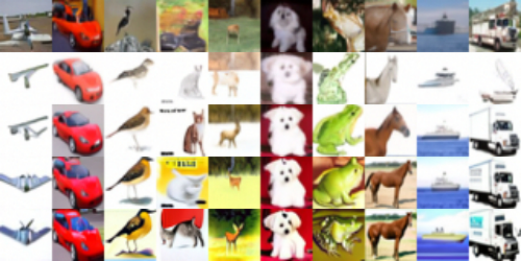}
	\caption{Visualization of the final generation results in \pmb{$\Delta=-0.03$}, \textbf{DDIM} and \textbf{U-net} on CIFAR-10 dataset.The top row of the image illustrates the ideal state where CDM generates images without any error. Subsequent 2$\sim$5 rows, demonstrate the images produced by RCDM with \( w \) set to 0, 5, 10, and 20, respectively.} 
	\label{fig.24}
\end{figure}

\begin{figure}[!ht]
	\centering
	\includegraphics[width=1.0\textwidth]{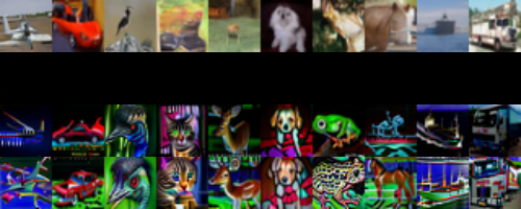}
	\caption{Visualization of the final generation results in \pmb{$\Delta=0.3$}, \textbf{DDPM} and \textbf{U-net} on CIFAR-10 dataset.The top row of the image illustrates the ideal state where CDM generates images without any error. Subsequent 2$\sim$4 rows, demonstrate the images produced by RCDM with \( w \) set to 0, 50, and 100, respectively.} 
	\label{fig.25}
\end{figure}

\begin{figure}[!ht]
	\centering
	\includegraphics[width=1.0\textwidth]{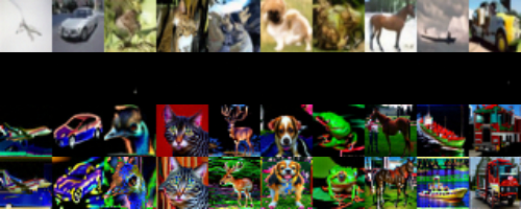}
	\caption{Visualization of the final generation results in \pmb{$\Delta=0.3$}, \textbf{DDPM} and \textbf{U-vit} on CIFAR-10 dataset.The top row of the image illustrates the ideal state where CDM generates images without any error. Subsequent 2$\sim$4 rows, demonstrate the images produced by RCDM with \( w \) set to 0, 50, and 100, respectively.} 
	\label{fig.26}
\end{figure}

\begin{figure}[H] 
	\centering
	\includegraphics[width=1.0\textwidth]{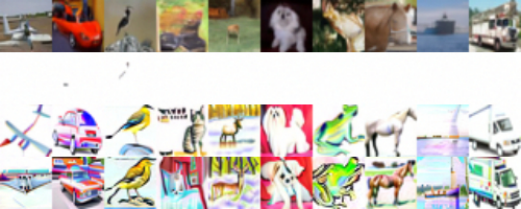}
	\caption{Visualization of the final generation results in \pmb{$\Delta=-0.3$}, \textbf{DDIM} and \textbf{U-net} on CIFAR-10 dataset.The top row of the image illustrates the ideal state where CDM generates images without any error. Subsequent 2$\sim$4 rows, demonstrate the images produced by RCDM with \( w \) set to 0, 50, and 100, respectively.} 
	\label{fig.27}
\end{figure}

\begin{figure}[H]
	\centering
	\includegraphics[width=1.0\textwidth]{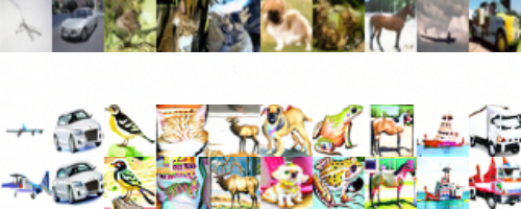}
	\caption{Visualization of the final generation results in \pmb{$\Delta=-0.3$}, \textbf{DDIM} and \textbf{U-vit} on CIFAR-10 dataset.The top row of the image illustrates the ideal state where CDM generates images without any error. Subsequent 2$\sim$4 rows, demonstrate the images produced by RCDM with \( w \) set to 0, 50, and 100, respectively.} 
	\label{fig.28}
\end{figure}

\subsection{Limitations}
While the advancements in RCDM have provided some resistance to errors for generative tasks. However,  there are several intrinsic challenges and limitations that need to be addressed. These challenges not only pertain to the technical aspects of model development but also to the practical considerations of deploying these models in real-world applications. These include refining the quantification of input errors, addressing the impact of stochastic noise, and balancing the trade-off between output diversity and model robustness. In image generation, bolstering network robustness might reduce output diversity, which could impair the task's goal of variety. However, this trade-off may be tolerable in prediction tasks where diversity is less crucial.

\subsection{Broader Impact}
Advancements in robust conditional diffusion models offer significant advantages. Firstly, these models are capable of dynamically adjusting $w$ based on errors, enhancing their robustness. Secondly, they have practical implications for improving the usability and generalization of conditional diffusion models. The potential for societal impact is substantial, as these models can reshape access to high-quality generative AI, lower computational barriers, and drive innovation across creative disciplines.
Ethical guidelines and verification methods are essential to ensure the technology's positive contributions, prevent misuse, and foster new job opportunities centered around AI-assisted creativity. Balancing these aspects is key to enabling the technology's constructive influence on various sectors and the broader community.

\subsection{Safeguards}
We are dedicated to safeguarding the personal information with robust encryption and stringent access controls, ensuring the data remains secure and confidential.
Our commitment to global data protection standards and continuous security enhancements upholds the trust you place in us.

\end{document}